\documentclass[conference]{IEEEtran}
\IEEEoverridecommandlockouts
\usepackage{cite}
\usepackage{amsmath,amssymb,amsfonts}
\usepackage{algorithmic}
\usepackage{graphicx}
\usepackage{textcomp}
\usepackage{xcolor}
\usepackage[mathscr]{eucal}
\usepackage{amsbsy,mathrsfs}

\include{macrofile2}

\usepackage{fancyhdr}
\usepackage{hyperref}
\usepackage{microtype}
\usepackage{amssymb}
\usepackage{amsmath}
\usepackage{booktabs}       
\usepackage{amsfonts}       
\usepackage{nicefrac}       
\usepackage{microtype}      
\usepackage[utf8]{inputenc}
\usepackage{multirow}
\usepackage{float}
\usepackage[thinc]{esdiff}

\usepackage{amsmath,amssymb,bm}
\usepackage{algorithm}
\usepackage{algorithmic}
\usepackage{arydshln}
\usepackage{multicol}
\usepackage{verbatim}
\usepackage{tikz}
\usepackage{longtable}
\usepackage{amsthm}
\usepackage{longtable}

\usepackage{comment}

\newtheorem{theorem}{Theorem}

\newtheorem{lemma}[theorem]{Lemma}

\def\tuple{\ensuremath{{(x,x^+,x^-_1,\ldots,x^-_m)}}} 

\def\eE{\ensuremath{{\mathbb{E}}}}
\def\rR{\ensuremath{{\mathbb{R}}}}
\def\sS{\ensuremath{{\mathbb{S}}}}

\def\Fcal{\ensuremath{{\mathcal{F}}}}
\def\Pcal{\ensuremath{{\mathcal{P}}}}
\def\Xcal{\ensuremath{{\mathcal{X}}}}

\def\Pneg{\ensuremath{{P_{\mathrm{neg}}}}}
\def\Psim{\ensuremath{{P_{\mathrm{sim}}}}}
\def\Pmar{\ensuremath{{P_{\mathrm{mar}}}}}

\def\frob{\ensuremath{{f_{\mathrm{rob}}}}}
\def\Omegamar{\ensuremath{{\Omega_{\mathrm{mar}}}}}

\allowdisplaybreaks

\usepackage{xcolor}
\usepackage{bbm}

\usepackage{newfloat}
\usepackage{listings}
\newfloat{listing}{tb}{lst}{}
\floatname{listing}{Listing}

\def\BibTeX{{\rm B\kern-.05em{\sc i\kern-.025em b}\kern-.08em
    T\kern-.1667em\lower.7ex\hbox{E}\kern-.125emX}}

\begin{document}

\title{Hard Negative Sampling via Regularized Optimal Transport for Contrastive Representation Learning}

\author{\IEEEauthorblockN{1\textsuperscript{st} Ruijie Jiang}
\IEEEauthorblockA{\textit{dept. of ECE} \\
\textit{Tufts University}\\
Medford, USA \\
Ruijie.Jiang@tufts.edu}
\and
\IEEEauthorblockN{2\textsuperscript{nd} Prakash Ishwar}
\IEEEauthorblockA{\textit{dept. of ECE} \\
\textit{Boston University}\\
Boston, USA \\
pi@bu.edu}
\and
\IEEEauthorblockN{3\textsuperscript{rd} Shuchin Aeron}
\IEEEauthorblockA{\textit{dept. of ECE} \\
\textit{Tufts University}\\
Medford, USA\\
Shuchin@ece.tufts.edu}
}

\maketitle

\fancyfoot[L]{\footnotesize \textbf{Left Footer Text}} 
\fancyfoot[C]{\footnotesize \textbf{Center Footer Text}} 
\fancyfoot[R]{\footnotesize \href{https://www.example.com}{Link Text}} 

\begin{abstract}
We study the problem of designing hard negative sampling distributions for unsupervised contrastive representation learning. We propose and analyze a novel min-max framework that seeks a representation which minimizes the maximum (worst-case) generalized contrastive learning loss over all couplings (joint distributions between positive and negative samples subject to marginal constraints) and prove that the resulting min-max optimum representation will be degenerate. This provides the first theoretical justification for incorporating additional regularization constraints on the couplings. We re-interpret the min-max problem through the lens of Optimal Transport (OT) theory and utilize regularized transport couplings to control the degree of hardness of negative examples. 
Through experiments we demonstrate that the negative samples generated from our designed negative distribution are more similar to the  anchor than those generated from the baseline negative distribution.
We also demonstrate that entropic regularization yields negative sampling distributions with parametric form similar to that in a recent state-of-the-art negative sampling design and has similar performance in multiple datasets. 
Utilizing the uncovered connection with OT, we propose a new ground cost for 
designing the negative distribution  and show improved performance of the learned representation on downstream tasks compared to the representation learned when using squared Euclidean cost.\footnote{Our code is publicly available at \href{https://github.com/rjiang03/HCL-OT}{https://github.com/rjiang03/HCL-OT}.} 
\end{abstract}


\section{Introduction}
In the unsupervised contrastive representation learning problem, we are given pairs of samples that are \textit{similar} to each other, e.g., an image and its augmentation, together with a set of dissimilar examples for each pair. One sample in a pair of similar samples is referred to as the \textit{anchor} and the other as the corresponding \textit{positive sample}. The dissimilar samples are referred to as the \textit{negative samples}. The goal is to learn a (typically) low-dimensional \textit{representation map} that amplifies the similarities and dissimilarities in the representation space by optimizing a suitable \textit{contrastive loss function}.  
The utility of learning this representation appears in a downstream task, where one can simplify both the model complexity, e.g., use a linear classifier, as well as reduce the number of training examples needed to achieve a desired classification error. A good set of representative papers for contrastive representation learning  in this context are \cite{oord2018representation, Zhu:2020vf}. 

{Several contrastive loss functions have been proposed in the literature} e.g., \cite{oord2018representation, wang2020understanding}. {Often,} the set of similar examples are determined using the \textit{domain knowledge}, e.g.,  SimCLR \cite{chen2020simple} uses augmented views of the same images as positive pairs for image data, \cite{perozzi2014deepwalk} use random walks on the graph to generate positive pairs for graph data, Word2Vec \cite{goldberg2014word2vec} and {Quick Thoughts} \cite{logeswaran2018efficient} define the neighborhood of words and sentences as positive for text data. On the other hand, the choice of negative samples, possibly conditioned on the given similar pair, remains an open design choice. It is well-known that this choice can theoretically \cite{saunshi2019theoretical, chuang2020debiased} as well as empirically \cite{tschannen2019mutual, jin2018unsupervised} affect the performance of contrastive learning. 

The main focus of this paper is to offer a novel theoretical and algorithmic  perspective on the design on negative sampling for unsupervised contrastive learning.

\section{Related work}

Several recent papers  \cite{robinson2021contrastive, kalantidis2020hard, ho2020contrastive, cherian2020representation} have revisited the design of negative sampling distributions and, in particular, ``hard'' negative and/or adversarially designed negative samples, an idea that originates in a related problem of metric learning \cite{weinberger2009distance, sohn2016improved}. Hard negative samples are samples that provide the ``most'' contrast to a given set of similar samples for a given representation map, while possibly trading off inadvertent \textit{hard class collisions in a downstream classification task} that are inevitable due to the unsupervised nature of the problem. In a similar vein, \cite{wu2021conditional}  proposes to use neither ``too hard'' nor ``too easy'' negative samples via predefined percentiles and \cite{zheng2021contrastive} makes use of ``learnable'' conditional distributions according to how informative the samples are to the anchor. 

The present work is most closely related to and motivated by the work in \cite{robinson2021contrastive} which proposed a simple exponential-form hard negative sampling distribution with a scalar exponential tilting parameter $\beta$ and developed a practical, low-complexity importance sampling strategy that accounts for the lack of true dissimilarity information. Under an oracle assumption that the support of the hard negative sampling distribution is disjoint from samples with the same class as the anchor (which is impossible to guarantee in an unsupervised setting where we have no knowledge of downstream decision-making tasks), they establish certain desirable worst-case generalization properties of their sampling strategy in the asymptotic limit $\beta \rightarrow \infty$. They also empirically demonstrate the performance improvements offered by their proposed sampling method in downstream tasks through extensive experiments on image, graph and text data.
In contrast to the heuristically motivated exponential form of the hard negative sampling distribution in \cite{robinson2021contrastive}, we develop a principled minimax framework for hard negative sampling based on regularized Optimal Transport (OT) couplings. The regularization can be systematically relaxed to effect a smooth control over the degree of hardness of negative samples.

\section{Main contributions}
This paper makes the following key contributions. 

\begin{itemize}
\setlength  
    \item In contrast to the current literature where the form of the negative sampling distribution is often heuristically motivated, albeit backed with good intuition and insight, we start with a principled (and novel) min-max perspective for sampling hard negatives. This formulation features negative sampling as an optimal (worst-case) coupling (a joint distribution subject to marginal constraints) between the anchor, positive, and negative samples. 
    \item We show that for a large class of contrastive loss functions, the resulting min-max formulation contains degenerate solutions. We believe this provides the first theoretical justification in the unsupervised setting (to the best of our knowledge) that selecting worst-case negative sampling distributions can produce degenerate solutions motivating the need for some type of regularization.
    \item We provide a framework for designing negative sampling distributions through the theory of regularized {OT}. We show that a negative sampling distribution having a parametric form similar to the one proposed in \cite{robinson2021contrastive} can be obtained as a special case of regularized {OT} with entropic regularization. We conduct several experiments using our entropy-regularized OT negative sampling framework to empirically corroborate the results of \cite{robinson2021contrastive}.
    \item Through the insight of regularized {OT}, we proposed a new cost function for the design of the hard negative distribution. With the designed negative distribution, the learned representation function shows consistently better performance on the downstream task when using a KNN classifier.
\end{itemize}

\section{Unsupervised Contrastive Learning}\label{sec:contrastive_learning}

Unsupervised contrastive representation learning aims to learn a mapping $f$ from a sample space $\Xcal$ to a representation space $\Fcal$ such that similar sample pairs stay close to each other in $\Fcal$ while dissimilar ones are far apart. We assume that the representation space is a subset of the Euclidean space ($\Fcal \subseteq \rR^{d}$) of low dimension and is normalized to have unit norm ($\Fcal = \sS^{d-1}$, the unit sphere in $\rR^d$). Our focus is on the setting in which the learner has access to {independent and identically distributed (IID)} $(m+2)$-tuples, where each tuple $\tuple$ consists of an anchor sample $x$ and, associated with it, a positive sample $x^+$ and $m$ negative samples $x^-_1, \ldots, x^-_m$. For each given anchor-positive pair $(x,x^+)$, the corresponding $m$ negative samples are conditionally IID with distribution $\Pneg(x^-|x,x^+)$.\footnote{For simplicity of exposition, we assume that all distributions are densities unless otherwise noted.} Thus, the joint distribution of the $(m+2)$-tuple is of the form: 
\begin{equation*}
P\tuple  = \Psim(x,x^+) \prod_{i=1}^m \Pneg(x^-_i|x,x^+). 
\end{equation*}
The distribution $\Psim(x,x^+)$ for generating anchor-positive pairs can capture a variety of situations. These include word-context pairs in NLP data, with words as anchors and their context-words as positive samples, or image-augmentation pairs in visual data, with images as anchors and their augmentations, e.g., rotation, cropping, resizing, etc., as positive samples. \\
Implicit to many applications is the assumption that the anchor, positive, and negative samples have the same marginal distribution $\Pmar$. This property also holds for the recently proposed latent ``class'' modeling framework of \cite{saunshi2019theoretical} for contrastive unsupervised representation learning which has been adopted by several works, e.g., \cite{robinson2021contrastive}. 

Let $\Pcal(\Pmar)$ denote the set of joint distributions $P$ having the form shown above 
with a common marginal distribution $\Pmar$ for the anchor, positive, and negative samples.

The representation $f$ is learned by minimizing the expected value of a loss function $\ell(f(x),f(x^+),f(x^-_1),\ldots,f(x^-_m))$ over all $f \in \mathcal{F}$. Some of the most widely used and successful loss functions are the triplet and logistic loss functions:

\noindent \textbf{Triplet-loss} \cite{schroff2015facenet}
\begin{flalign*}
&\ell_{\mathrm{triplet}}(x,x^+,x^-,f) \\
&= \max(0,||f(x)-f(x^+)||^2 - ||f(x)-f(x^-)||^2 + \eta)\\
&= \max(0,2\,v + \eta)
\end{flalign*}
where $v := f^\top(x)\,f(x^-) - f^\top(x)\,f(x^+)$, $m=1$, and $\eta > 0$ is a margin hyper-parameter. The second equality above holds because $\mathcal{F} = \sS^{d-1}$.

\noindent \textbf{Logistic-loss} \cite{gutmann2010noise} also referred to as the $m$-pair multiclass logistic loss or sometimes loosely as the Noise Contrastive Estimation (NCE) loss 
{\small  
\begin{flalign}
\!\!\!\ell_{NCE}(x,x^+,\{x^-_i\},\,f)  
&=  \log\left(1 + \frac{q}{m} \sum_{i=1}^m \frac{e^{f^\top(x)f(x^-_i)}}{e^{f^\top(x)f(x^+)}}\right) 
\nonumber \\
&=  \log\left(1 + \frac{q}{m} \sum_{i=1}^m e^{v_i}\right)  
\label{eq:NCE}
\end{flalign}}
\hspace{-1ex}where $v_i := f^\top(x)\,f(x^-_i) - f^\top(x)\,f(x^+)$ and $q > 0$ is a weighting hyper-parameter. The second equality above holds because $\mathcal{F} = \sS^{d-1}$.

\noindent \textbf{General convex non-decreasing loss}: We consider a general class of loss functions of the following form which includes the triplet and logistic loss functions:
\begin{flalign*}
\ell(x,x^+,\{x^-_i\},\,&f) =  \psi\left(v_1,\ldots,v_m\right)
\end{flalign*}
where $v_i := f^\top(x)\,f(x^-_i) - f^\top(x)\,f(x^+)$ and $\psi(\cdot)$ is a convex function of $(v_1,\ldots,v_m)$ which is argument-wise non-decreasing, i.e., non-decreasing with respect to $v_i$ for each $i$ at any $(v_1,\ldots,v_m)$. The optimal representation map is the $f^* \in \Fcal$ which minimizes the expected loss:
\begin{flalign*}
f^* = \arg\min_{f \in \Fcal} \eE_{P}[\psi(v_1,\ldots,v_m)]
\end{flalign*}
where the expectation is taken with respect to the joint distribution $P \in \Pcal(\Pmar)$ of $\tuple$. In practice, the expectation is replaced by an empirical average over IID training tuples with each tuple consisting of an anchor, a positive sample, and $m$ negative samples.

\section{Min-max contrastive learning} \label{sec:minimax_representation}

How does one design the negative sampling distribution $\Pneg$? This is a key question in unsupervised contrastive representation learning. In early works \cite{oord2018representation, chen2020simple}, the negative samples were generated \textit{independently} of the anchor and positive samples according to the marginal distribution $\Pmar$, i.e., $\Pneg(x^-|x,x^+) = \Pmar(x^-)$. Today we regard this as a baseline method for negative sampling. In later works \cite{robinson2021contrastive,wu2021conditional,zheng2021contrastive}, researchers noted that ``hard'' or ``semi-hard'' negative samples, i.e., samples that are near the anchor in representation space, {but far from it} (i.e., dissimilar) in sample space could guide learning algorithms to correct ``mistakes'' in downstream supervised tasks more quickly \cite{kalantidis2020hard}. A variety of strategies for sampling hard negatives have been devised in the metric learning literature \cite{xiong2021approximate}. They essentially aim to select samples that are difficult to discriminate based on the current representation, often leveraging ideas from importance sampling \cite{robinson2021contrastive}. 

In contrast, we propose to generate hard negatives by designing a worst-case $\Pneg$ distribution subject to suitable regularity constraints. We capture these constraints via $\Omega$, a family of $\Pneg$ that satisfy the regularity constraints. For a given $\Psim$ and $\Omega$, a robust representation $\frob$ is one which minimizes the maximum, i.e., worst-case, loss over $\Omega$:
\begin{flalign*}
\frob(\Omega) = \arg\min_{f \in \Fcal} \max_{\Pneg \in \Omega} \eE_P[\ell(x,x^+,\{x^-_i\})].
\end{flalign*}
A natural first candidate for $\Omega$ that may come to mind is the set of all $\Pneg$ that are marginally consistent with $\Pmar$, i.e., 
\[
\Omegamar := \{\Pneg:  \Psim \otimes_{i=1}^m \Pneg \in \Pcal(\Pmar)\}.
\]
Note that since the anchor and positive sample are assumed to have the same marginal, $\Pmar$ is completely specified by $\Psim$. Unfortunately, it turns out that $\frob(\Omegamar)$ contains degenerate solutions, i.e., solutions such that with probability one, $\frob(\Omegamar)(x) = c$ a constant vector in $\sS^{d-1}$. We state this negative result as a theorem below.
\begin{theorem}\label{thm:degenerate}
For the general convex argument-wise non-decreasing loss described in Sec.~\ref{sec:contrastive_learning}, the min-max optimum representation for the family of marginally consistent negative sampling distributions contains degenerate solutions with probability one, i.e., solutions such that 
\[
w.p.1, \quad \frob(\Omegamar)(x) = c
\]
a constant vector in $\sS^{d-1}$.
\end{theorem}
To prove Theorem~\ref{thm:degenerate} we will make use of the following lemma.
\begin{lemma}
\label{lem:degenerate} 
For any given $\Psim$,
\[
\min_{f\in\Fcal}\max_{\Pneg\in \Omegamar} \eE_{P}[f^\top(x)\,f(x^-) - f^\top(x)\,f(x^+)] = 0
\]
and equality can be attained for all $\Psim$ by a degenerate $f$, i.e., a constant representation map.
\end{lemma}
\begin{proof} For each $f \in \Fcal$,
\begin{flalign*}
\eE_{P}[f^\top(x)\,f(x^-)] &\leq \sqrt{\eE_{\Pmar}||f(x)||^2\,\eE_{\Pmar}||f(x^-)||^2} = 1
\end{flalign*}
by the Cauchy-Schwartz inequality because $\Fcal = \sS^{d-1}$. Equality can be attained if, and only if, $f(x) = f(x^-)$ with probability one. One choice of $\Pneg$ which ensures this is $x^- = x$ with probability one, independent of $x^+$. If $f$ is not injective, this is not the only choice. More generally, if the conditional distribution of $x^-$ given $x$ is equal to the conditional distribution of $x$ given its representation $f(x)$, then $x^-$ will have the same marginal distribution as $x$, i.e., $\Pmar$ since its generation is distributionally indistinguishable from the generation of $x$ in  two steps: first generate a representation $f(x)$ and then sample $x$ from the pre-image of $f(x)$ in $\Xcal$. Moreover, by construction, given $x$, $f(x^-) = f(x)$ with probability one. Thus, 
\begin{align}
&\max_{\Pneg\in \Omegamar} \eE_{P}[f^\top(x)\,f(x^-) - f^\top(x)\,f(x^+)]  \nonumber \\ 
&= \max_{\Pneg\in \Omegamar} \eE_{P}[f^\top(x)\,f(x^-)] -  \eE_{\Psim}[f^\top(x)\,f(x^+)]  \nonumber \\
&= 1 -  \eE_{\Psim}[f^\top(x)\,f(x^+)] \label{eq:maxoverPneg}
\end{align}
{
An application of the Cauchy-Schwartz inequality a second time gives us
\begin{align*}
\eE_{\Psim}\big[f^\top(x)\,f(x^+)] &\leq \eE_{\Psim}[||f(x)||\cdot||f(x^+)||\big] = 1
\end{align*}
where the last equality holds because $\mathcal{F} = \sS^{d-1}$.} Equality can be attained in the inequality simultaneously for all $\Psim$ if $f$ is a constant representation map. From this and Equation~(\ref{eq:maxoverPneg}) it follows that 
\begin{align*}
&\min_{f\in\Fcal}\max_{\Pneg\in \Omegamar} \eE_{P}[f^\top(x)\,f(x^-) - f^\top(x)\,f(x^+)] \\
&= 1 -  \max_{f\in\Fcal}\eE_{\Psim}[f^\top(x)\,f(x^+)] \geq 1 - 1 =0
\end{align*}
and equality can be attained in the inequality for all $\Psim$ if $f$ is a constant representation map.
\end{proof}
The proof of Theorem~\ref{thm:degenerate} now follows from the convexity and argument-wise non-decreasing properties of the loss function. By Jensen's inequality for convex functions we have,
\begin{align*}
&\eE_P[\psi(v_1,\ldots,v_m)] \geq \psi(\eE_P[v_1],\ldots,\eE_P[v_m])   \\
&= \psi(\eE_P[v],\ldots,\eE_P[v])
\end{align*}
where $v(x,x^+,x^-) := f^\top(x)\,f(x^-) - f^\top(x)\,f(x^+)$. This is because  all negative samples are IID conditioned on $x,x^+$.
From Lemma~\ref{lem:degenerate}, we have
\[
\min_{f\in\Fcal} \max_{\Pneg\in\Omegamar} v(x,x^+,x^-) = 0. 
\]
Since $\psi(v_1,\ldots,v_m)$ is argument-wise non-decreasing, it follows that 
\begin{align*}
&\min_{f\in\Fcal} \max_{\Pneg\in\Omegamar}\eE_P[\psi(v_1,\ldots,v_m)]  \\
&\geq
\psi(\min_{f\in\Fcal} \max_{\Pneg\in\Omegamar}\eE_P[v],\ldots,\min_{f\in\Fcal} \max_{\Pneg\in\Omegamar}\eE_P[v]) \\
&= \psi(0,\ldots,0).
\end{align*}
and equality can be attained by a degenerate representation which maps all samples to a constant vector in $\sS^{d-1}$. This concludes the proof of Theorem~\ref{thm:degenerate}.

The negative result of Theorem~\ref{thm:degenerate} motivates the need to incorporate other regularization constraints into $\Omega$ in order to preclude degenerate solutions. To develop this, in the next section we will interpret the min-max loss using the lens of {OT} and then utilize regularized transport couplings to design an optimal $\Pneg$.

\section{An Optimal Transport Perspective}
We begin with the necessary background on Optimal Transport (OT). Given two probability distributions $P(x)$ and $P(y)$ over $\mathbb{R}^d$, let $\Pi$ denote the set of joint distributions or couplings $P(x,y)$ with marginals $P(x)$ and $P(y)$. The problem of optimal transport (OT) \cite{santambrogio2015optimal} is to seek the optimal coupling that solves 
\begin{align}
    \mathsf{OT}(P(x), P(y), c)=  \min_{P(x,y) \in \Pi} \mathbb{E}_{P(x,y)}[c(x,y)],
\end{align}
where $c(x,y)$ referred to as the ground-cost, {is} the cost of \textit{transporting} $x$ to $y$. Existence of the solutions to this problem is guaranteed under fairly general assumptions on the ground-cost \cite{villani2021topics}.

To see how the OT set-up arises in the problem at hand, it is useful to focus {on} the case where we have a triplet $x,x^+,x^-$ and revisit the core Lemma \ref{lem:degenerate} and manipulate the objective like so
{
\small  
  \setlength{\abovedisplayskip}{3pt}
  \setlength{\belowdisplayskip}{\abovedisplayskip}
  \setlength{\abovedisplayshortskip}{0pt}
  \setlength{\belowdisplayshortskip}{3pt}
\begin{align}
& \min_{f\in\Fcal}\max_{\Pneg\in \Omegamar} \eE_{P}[f^\top(x)\,f(x^-) - f^\top(x)\,f(x^+)] \notag\\
= & \min_{f \in \Fcal}  \max_{P(x,x^-)\in \Pi} \eE_{P(x,x^-)}[f^\top(x)\,f(x^-)] - \eE_{\Psim}[f(x)^\top f(x^+)], \notag 
\end{align}}
where $\Pi$ is the set of couplings between $x,x^-$ with marginals $\Pmar$. Under this constraint, {noting that $\mathcal{F} = \sS^{d-1}$} and completing the square, we have that,
\begin{align}
    & \max_{P(x,x^-) \in \Pi} \eE{_{P(x,x^-)}}[f(x)^\top f(x^-)] \notag \\
    & = -0.5 \min_{P(x,x^-) \in \Pi} \eE{_{P(x,x^-)}}[\|f(x) -  f(x^-)\|_2^2] +   1. \notag
\end{align}
{We observe} that $\min_{P(x,x^-) \in \Pi} \eE{_{P(x,x^-)}}[\|f(x) -  f(x^-)\|_2^2] = \mathsf{OT}(P(f(x)), P(f(x^-)), c)$ with $c(x,x^-) = 0.5 \|f(x) - f(x^-) \|_2^2$.  {Since $x,x^-$ have the same marginals, it follows that $f(x)$ and $f(x^-)$ will have the same distributions and therefore $\mathsf{OT}(P(f(x)), P(f(x^-)), c) = 0$.} Hence, the min-max problem boils down to 
\begin{align}
\min_{f \in \Fcal} \big( 1 - \eE_{\Psim}[f(x)^\top f(x^+)] \big),
\end{align}
{which as we previously discussed has a degenerate solution where $f$ is a constant mapping and the optimal value is zero.}

In order to avoid the degeneracy that implicitly arises due to the OT coupling, {it would be desirable for $P(x,x^-)$ to satisfy the following properties.} 

\begin{itemize}
\setlength 
\item[1)] {\textbf{Non self-coupling:} for any triplet $(x, x^-,x^+)$, $x^- \neq x$ and $x^- \neq x^+$.}  
\item[2)] \textbf{{Product-coupling-improving:}} It is desirable that $P(x,x^-)$  {improves over the default coupling where $x$ and $x^-$ are IID, i.e.,}
\begin{align}
    \mathop{\eE}_{ \Pmar(x) \Pmar(x^-)} [f(x)^\top   f(x^-)] \leq \mathop{\eE}_{P(x,x^-) } [f(x)^\top   f(x^-)].\notag
\end{align}
\end{itemize}
In the next section we will show that these properties can be realized within the framework of \textit{regularized} OT {which also enables tuning the hardness of the designed negative distribution.}

\subsection{Regularized OT for negative sampling}
We consider the following regularized version of optimal transport \cite{genevay_thesis, peyre2019computational} with ground-cost $c(x,x^-) = 0.5 \| f(x) - f(x^-)\|^2$ if $x^- \neq x$ or $x^- \neq x^+$  and $\infty$ otherwise to avoid self-coupling 
{\footnotesize
\begin{align}
    &P^*(x,x^-) \nonumber \\
    &=\arg\!\!\!\!\!\!\!\!\!\min_{P(x,x^-) \in \Pi} \!\!\!\! \eE[c(x,x^-)] + \epsilon\, \varphi
    \left(P(x,x^-), \Pmar(x)\Pmar(x^-)\right),
    \label{Entrop_ot}
\end{align}
}
\hspace{-1ex}where $\varphi(u, v)$ is a convex function of $(u,v)$. 
The choice $\varphi = \mathsf{KL}(\cdot||\cdot)$, where 
\begin{align}
& \mathsf{KL}(P(x,x^-)|| \Pmar(x)\Pmar(x^-))  \notag\\ = &  \int_x \int_{x^-} P(x,x^-) \log \frac{P(x,x^-)}{\Pmar(x)\Pmar(x^-)}  dxdx^-,     
\end{align} 
is the Kullback-Liebler divergence, leads to what is referred to as entropy-regularized OT which can be implemented in a computationally efficiently manner using the Sinkhorn algorithm \cite{peyre2019computational}. The parameter $\epsilon$ is referred to as the entropic regularization parameter. For the $\mathsf{KL}$ choice we note the following properties:
\begin{enumerate}
\setlength 
    \item The hardness of negative samples decreases as $\epsilon$ increases from zero to infinity.
    \item Since $\mathsf{KL} \geq 0$ with equality iff $P(x,x^-) = \Pmar(x)\Pmar(x^-)$, it can be verified that
    {\small  
  \setlength{\abovedisplayskip}{3pt}
  \setlength{\belowdisplayskip}{\abovedisplayskip}
  \setlength{\abovedisplayshortskip}{0pt}
  \setlength{\belowdisplayshortskip}{3pt}
    \begin{align}
    \mathop{\eE}_{ \Pmar (x) \Pmar(x^-)} [f(x)^\top   f(x^-)] \leq \mathop{\eE}_{P^*(x,x^-) } [f(x)^\top   f(x^-)] \notag
\end{align}}
\end{enumerate}
\textbf{Connection to \cite{robinson2021contrastive}} We now show how the optimal solution to Equation~(\ref{Entrop_ot})  justifies the empirical choice made in \cite{robinson2021contrastive}. To this end, we note the following result.
\begin{lemma} \cite{genevay_thesis} \label{connection}
There exist continuous functions $u(x), v(x^-)$ such that the optimal entropy-regularized coupling is given by,
\begin{align}
     P^*(x,x^-) \notag = \Pmar(x)\,\Pmar(x^-)\, e^{\frac{u(x) + v(x^-) - c(x,x^-)}{\epsilon}} \notag
\end{align}
\end{lemma}
Specializing this result to our case we note that 
\begin{align}
     P^*(x,x^-) = \Pmar(x)\,\Pmar(x^-)\,e^{\frac{u(x) + v(x^-) +  f(x)^\top f(x^-) - 1}{\epsilon}}  \notag,
\end{align}
which implies that, 
    \begin{align}
    & P^*(x^-|x) \propto \Pmar(x^-)\, e^{\frac{f(x)^\top f(x^-)}{\epsilon}}  \label{eq:opt_ent_reg_coupling}
\end{align}
We now recall that in \cite{robinson2021contrastive}, the authors chose $P(x^-|x) \propto \Pmar(x^-)\,e^{\beta f(x)^\top f(x^-)} $ for the design of hard-negatives which is of the similar exponential form as $ P^*(x^-|x)$ given by Equation~\eqref{eq:opt_ent_reg_coupling}.\\
{In this context we would like to point out that if one removes the marginal constraints captured by $\Pi$ in \eqref{Entrop_ot} with KL divergence as the regularizer then 
\begin{align*}
    & \arg\min_{P(x,x^-)} \!\!\!\! \eE[c(x,x^-)] + \epsilon \mathsf{KL}(P(x,x^-)|| \Pmar(x)\Pmar(x^-)) \\
    & = \arg\min_{P(x,x^-)} \!\!\!\! \mathbb{E}\Bigg[ \log \Bigg( \frac{P(x,x^-)}{e^{- \frac{c(x,x^-)}{\epsilon}}\Pmar(x) \Pmar(x^-)}\Bigg)\Bigg].
\end{align*}
Noting that for the case at hand, $c(x,x^-) = 1 - f(x)^\top f(x^-)$, the optimal solution satisfies $P(x,x^-) \propto \Pmar(x) \Pmar(x^-) \, e^{ \beta f(x)^\top f(x^-)}$ with $\beta = \frac{1}{\epsilon}$ which implies $P(x^-|x) \propto \Pmar(x^-) \, e^{\beta f(x)^\top f(x^-)} $, which corresponds to the design of negative sampling proposed in \cite{robinson2021contrastive}. 
}

\section{Experimental Corroboration}

In the previous section, we have shown that our entropy-regularized OT solution  has an exponential form similar to the hard negative sampling distribution of  \cite{robinson2021contrastive} which they refer to as HCL. In this section, we present experiments on four image and five graph datasets to empirically demonstrate that our proposed method for negative sampling can indeed match (and sometimes slightly exceed) the state-of-the-art performance results reported in \cite{robinson2021contrastive}. For these experiments, we adopt the same simulation set-up as in HCL \cite{robinson2021contrastive}.
We use their implementation code but replace their negative coupling with the obtained by solving Equation~\eqref{Entrop_ot} with $\varphi = \mathsf{KL}$.

For entropy-regularized OT, we use the POT package \cite{flamary2021pot} to implement the numerically stable version of the entropy-regularized OT algorithm in \cite{cuturi2013sinkhorn}. 

\subsection{Image datasets}
We work with four image datasets: STL10 \cite{coates2011analysis}, CIFAR10, CIFAR100 \cite{krizhevsky2009learning}, and Tiny-ImageNet \cite{le2015tiny}, each containing images with 10, 10, 100, and 200 classes, respectively. To evaluate our method, we use SimCLR \cite{chen2020simple} as a baseline and compare it to the algorithm proposed in \cite{robinson2021contrastive} using three different loss functions to learn the mapping $f$ with negative samples. These include the large-$m$ asymptotic form of the NCE loss ($m\rightarrow \infty$) \cite{robinson2021contrastive} given by 
{\footnotesize
\begin{align}
\!\!\!\!\mathcal{L}_{NCE} = \eE_{\Psim(x,x^+)}\Bigg[\log\Bigg(1 + \frac{o \cdot \eE_{P^*(x^-|x)}[e^{f(x)^\top f(x^-)}]}{e^{f(x)^\top\,f(x^+)}}\Bigg)\Bigg] \label{NCE_loss_2}
\end{align}
}
\hspace{-1ex}where $o$ is the weighting parameter, the debiased NCE loss  \cite{chuang2020debiased}, and the Basic Contrastive Loss (BCL) associated with the objective of Lemma~\ref{lem:degenerate}, i.e., 
\begin{align}
\label{eq:upper}
    \mathcal{L}_{BCL} = \mathbb{E}[f^\top(x) f(x^-) - f^\top(x) f(x^+)].
\end{align}
We report classification accuracy for representations learned by optimizing these three loss functions for the STL10, CIFAR10, and CIFAR100 datasets. For Tiny-ImageNet, we only report results based on the NCE loss in Equation~(\ref{NCE_loss_2}) as it is a large dataset. In our experiments, we approximate all expectations appearing in the loss functions using empirical averages over a batch of $B$ training samples. 
For the representation function we use the modified ResNet-50 architecture in  \cite{robinson2021contrastive}
with a representation dimension of 2048 and a projection head with a dimension of 128 used in  \cite{robinson2021contrastive}. 
For linear reasdout, we train a linear classifier with a fixed representation generated by the learned ResNet-50.

\begin{table}[tbp]
\caption{\label{tbl:STL10}
{Best linear readout accuracy on STL10 dataset.}}
\begin{center}
\begin{tabular}{|c c c c|}
\hline
 Loss function &  SimCLR \cite{chen2020simple}& HCL\cite{robinson2021contrastive}  & Entropic OT  \\ 
 & & & (best $\epsilon$)\\
\hline
 $\mathcal{L}_{BCL}$  &None&83.0&82.8 ($\epsilon = 0.3$)\\ 
 \hline
 NCE  &80.2&84.4&85.0 ($\epsilon = 0.3$)\\
 \hline
 Debiased NCE \cite{chuang2020debiased} &84.3&87.4&87.5 ($\epsilon = 0.3$)\\
 \hline
\end{tabular}
\end{center}

\end{table}

\begin{table}[tbp]
\caption{\label{tbl:CIFAR10}
{Best linear readout accuracy} on CIFAR10 dataset.}
\begin{center}
\begin{tabular}{|c c c c|}
\hline
 Loss function &  SimCLR \cite{chen2020simple}& HCL\cite{robinson2021contrastive}  & Entropic OT \\ 
  & & & (best $\epsilon$)\\
\hline
 $\mathcal{L}_{BCL}$   &None&90.2&90.8 ($\epsilon = 0.7$)\\ 
 \hline
 NCE  &91.0&91.4&91.2 ($\epsilon = 0.5$)\\
 \hline
 Debiased NCE \cite{chuang2020debiased}&91.2&92.1&91.8 ($\epsilon = 0.7$)\\
 \hline
\end{tabular}
\end{center}

\end{table}

\begin{table}[tb]
\caption{\label{tbl:CIFAR100}
{Best linear readout accuracy on} CIFAR100 dataset.}
\begin{center}
\begin{tabular}{|c c c c|}
\hline
 Loss function &  SimCLR \cite{chen2020simple}& HCL\cite{robinson2021contrastive}  & Entropic OT \\ 
  & & & (best $\epsilon$)\\
\hline
 $\mathcal{L}_{BCL}$   &None&64.8&67.3 ($\epsilon = 0.5$)\\ 
 \hline
 NCE  &66.6&69.2&69.1 ($\epsilon = 0.7$)\\
 \hline
 Debiased NCE \cite{chuang2020debiased}&67.5 &69.3 &69.5 ($\epsilon = 0.7$)\\
 \hline
\end{tabular}
\end{center}

\end{table}

\textbf{Training Procedure:} 
 There are two hyperparameters to tune namely, $\epsilon$ for our method, and $\tau$ that appears as a hyper-parameter in the debaised loss \cite{chuang2020debiased}. For these we perform a grid search over the set $\{0.1,0.3,0.5,0.7,1\}$ for $\epsilon$ and the set $\{0.01, 0.05, 0.1, 0.5\}$ for $\tau$. We kept all other hyperparameters the same as in \cite{robinson2021contrastive} to ensure a fair comparison.

All models are trained for $E = 400$ epochs with batch size $B = 512$. We use the Adam optimizer with $0.001$ as learning rate and $10^{-6}$ as the weight decay value. We use NVIDIA A100 32 GB GPU for our computations and it takes about 20 hours to train one model (400 epochs) for each dataset.

\textbf{Results:} Tables~\ref{tbl:STL10}, \ref{tbl:CIFAR10}, \ref{tbl:CIFAR100}, and \ref{tbl:Tiny-imagenet} compare the best accuracies attained by different methods using three different loss functions on four image datasets. Our proposed method clearly improves over the baseline. For the NCE loss, we observe absolute improvements of $4.8$, $2.5$, and $2.9$ percentage points over SimCLR on the STL10, CIFAR100, and Tiny-imagenet datasets respectively. For Debiased NCE loss, the absolute improvements over SimCLR are $3.2$ and $2.0$ percentage points on the STL10 and CIFAR100 datasets, respectively. The performance of our method is very similar to that of \cite{robinson2021contrastive}. This is consistent with our theoretical analysis and insights from Lemma \ref{connection} which shows that our entropy-regularized hard negative sampling distribution has the similar exponential form as the distribution proposed in \cite{robinson2021contrastive}.

\begin{table}[htb]
\caption{\label{tbl:Tiny-imagenet}
{Best linear readout accuracy on} Tiny-ImageNet dataset.}
\begin{center}
\begin{tabular}{|c c c c|}
\hline
 Loss function &  SimCLR \cite{chen2020simple}& HCL\cite{robinson2021contrastive}  & Entropic OT \\ 
  & & & (best $\epsilon$)\\
 \hline
 NCE  &53.4&56.1&56.3  ($\epsilon = 0.5$)\\
 \hline
\end{tabular}
\end{center}

\end{table}

\setlength{\tabcolsep}{2pt}
\begin{table}[!htp] 
\caption{\label{tbl:graph} Accuracy on graph datasets.}
\begin{center}
\begin{tabular}{|c|ccccc|} 
\hline
 Method &  MUTAG & ENZYMES & PTC & IMDB-B & IMDB-M \\ \hline
 InfoGraph~\cite{sun2019infograph}  & 86.8 & 50.4 & 55.3 & 72.2 & 49.6 \\ \hline
 HCL~\cite{robinson2021contrastive} & 87.2 & 50.4 & $\bold{57.3}$ & 72.8 & 49.6 \\ \hline 
 Entropic OT, $\epsilon=0.05$  & 87.2 & $\bold{51.3}$ & 56.8 & $\bold{73.0}$ & 49.8 \\ \cline{2-6}
 Entropic OT, $\epsilon=0.1$  & $\bold{88.2}$ & 50.6&56.1 & 72.7& 49.8 \\ \cline{2-6}
 Entropic OT, $\epsilon=0.5$ & 87.5 & 50.9 & 55.8 & 72.1 & $\bold{50.0}$ \\ \cline{2-6}
 Entropic OT, $\epsilon=1$ & 87.6 & 50.3 & 55.4 & 72.5 & $\bold{50.0}$ \\ \hline
\end{tabular}
\end{center}
\end{table}

\subsection{Graph dataset}
We also appllied our method to learn graph representations on five graph datasets: MUTAG, ENZYMES, PTC, IMDB-BINARY, and IMDB-MULTI by \cite{morris2020tudataset}. We employ InfoGraph \cite{sun2019infograph} as a baseline method. 

\textbf{Training Procedure:} Similarly to the experiments with image datasets, we only replace the negative distribution in the code implementation from \cite{robinson2021contrastive} with ours. The dimension of the node embedding is set to $96$. 
We search for the best values of $\epsilon$ and $\tau$ over the sets $\{0.05,0.1,0.5,1\}$ and $\{0.1, 0.5\}$ respectively. 
For every choice of $\epsilon$ the best value of $\tau$ turned out to be $0.5$. 
We therefore report the accuracy for all choices of $\epsilon$ and $\tau=0.5$. All models are trained for 200 epochs and we use the Adam optimizer with learning rate $0.01$. We used the 3-layer GIN \cite{xu2018how} as the representation function with representation dimension equal to $32$. Then we train an SVM classifier based on the learned graph-embedding. Each model is trained 10 times with 10-fold cross-validation. We report the average accuracy in Table \ref{tbl:graph}.

\textbf{Results:} We report the performance accuracy of the different methods in Table~\ref{tbl:graph} with boldface numbers indicating the best performance. We observe that our method is consistently better than the baseline \cite{sun2019infograph} in all datasets improving the accuracy by $1.4$ and $1.3$ percentage points on the MUTAG and PTC datasets, respectively, and is competitive with the state-of-the-art method in \cite{robinson2021contrastive}.

\subsection{Regularized OT with different cost function}
One advantage of our proposed method is that regularized optimal transport allows us to utilize different cost functions $c(f(x), f(x^-)): \mathbb{R}^d\times\mathbb{R}^d \to \mathbb{R}_{\geq 0}$ with desirable properties compared to the default cost of squared Euclidean distance. For the squared Euclidean case, it has been noted in \cite{cai2020all} that samples that are too close to the anchor are often false negatives (having a class label different from the anchor). On the other hand samples that are too far from the anchor are uninformative. Thus, intuitively, for a negative sample $x^-$ that is too close or too far from the anchor $x$ in the representation space, the cost $c(f(x), f(x^-))$ should be greater than the squared Euclidean distance between $f(x)$ and $f(x^-)$. The corresponding negative distribution should then assign a higher weight to true hard negative samples (hard negative sample with a class label different from that of the anchor).

In this context we propose to use $c(f(x), f(x^-)) = e^{\|f(x)-f(x^-)\|^2_2-\kappa}$ where $\kappa$ is a hyperparameter that trades off the cost of coupling with the degree of entropic regularization. 
We plot the average Top 1 accuracy of the KNN classifier for every five epochs in Fig \ref{fig:new_cost}. The performance of our method with {the} proposed cost is consistently better than the one with {the} squared Euclidean cost.

\begin{figure}[t]

  \begin{minipage}[h]{0.5\linewidth}
    \centering
    \hspace{-6mm}
    \includegraphics[scale=0.35]{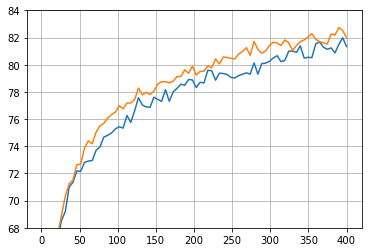}
  \end{minipage}%
  \begin{minipage}[h]{0.6\linewidth}
    \centering
    \hspace{-10mm}
    \includegraphics[scale=0.35]{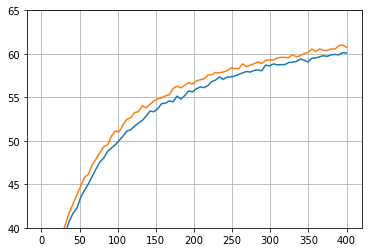}
  \end{minipage}
  \caption{Top 1 KNN accuracy vs epochs for STL10 (left) and CIFAR100 (right). The orange line represents the proposed cost while the blue line corresponds to our default cost function.}
\label{fig:new_cost}

\end{figure}

\begin{figure}[b]
\centerline{\includegraphics[width=9cm, height=3cm]{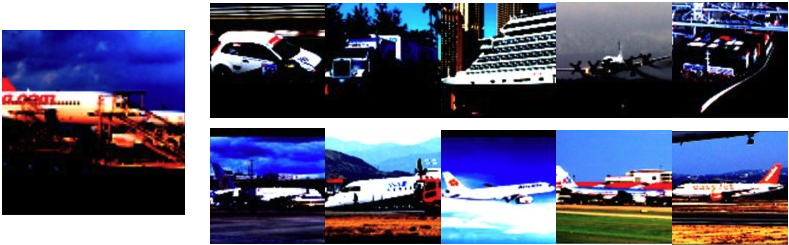}}
\caption{Left: Anchor image. Upper: Top-5 most similar images before training. Lower: Top-5 similar images with our method after training. }\label{10_sim}
\end{figure}

\begin{figure}[tb]
\centerline{\includegraphics[width=9cm, height=3cm]{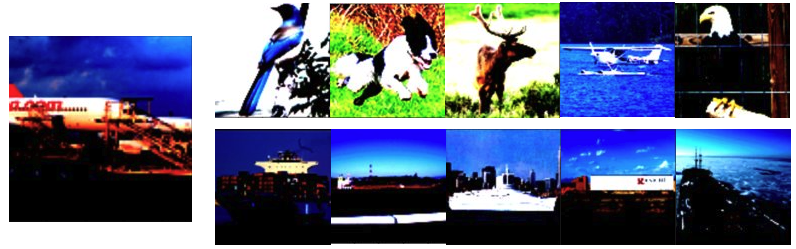}}
\caption{Left: Anchor image. Upper: Uniformly sampled negative images before training. Lower: Negative images with highest weight generated by our method before training. }\label{10_weight}
\end{figure}

\subsection{Discussion}
To glean qualitative insights into how our method works, in Figure~\ref{10_sim} we plot the top five most similar images to the anchor, i.e., the images with the highest inner product in representation space with the anchor image, 
before training and after training. The larger image on the left is the anchor and to its right are ten images from the same batch. We observe that before training, only one of these five images comes from the same class as the anchor. However, after training for 400 epochs, all five images come from the same class. This indicates that more images with the same labels are becoming closer to each other in the representation space, which is how we would expect a contrastive representation learning method to behave.

In Figure~\ref{10_weight} we show the negative images uniformly chosen from the first batch of the first epoch (i.e., before training commences) and the negative images with the highest $P^*(x^-|x)$ generated by our method. We observe that uniformly sampled negative images are semantically unrelated to the anchor (only 1 out of 5 negative images is related to transportation). But the negative images sampled from our designed ${P}^*(x^-|x)$ is more similar to the anchor in terms of color, background, and significantly, all 5 negative images are related to transportation. 

We plot $P^*(x^-|x)$ and compare it with ${f(x)^{\top}f(x^-)}$ in Figure \ref{trans_inner}. In this figure, the horizontal axis is ordered from left to right by decreasing values of $f(x)^{\top}f(x^-)$. The figure shows that the negative sample $x^-$ with larger $f(x)^{\top}f(x^-)$ is likely to have a higher $ {P}^*(x^-|x)$, which means our designed negative distribution assigns a higher weight to the negative sample closer to the anchor. From Figure~\ref{10_sim}, we see that as we are training the model, samples with the same label will become closer in the representation space (i.e., more similar). This means that after training for a certain number of epochs, the samples coming from the same class as the anchor are more likely to be selected as the negative samples.

\begin{figure}[t]

  \begin{minipage}[h]{0.5\linewidth}
    \centering
    \hspace{-4mm}
    \includegraphics[scale=0.12]{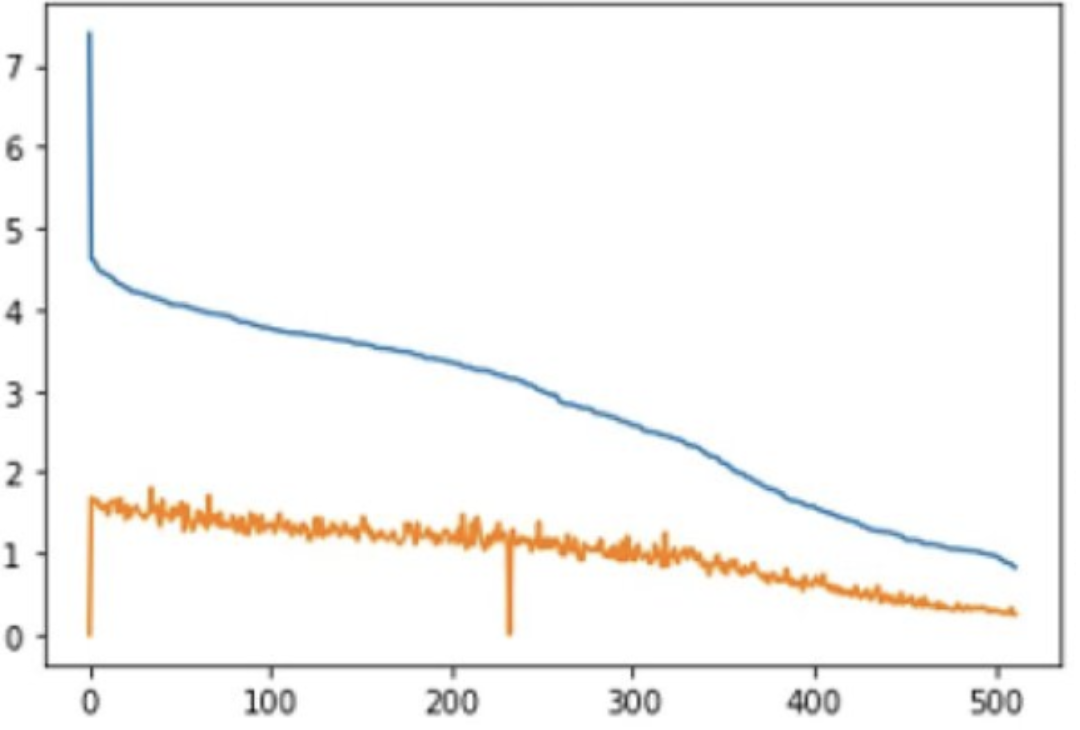}
  \end{minipage}%
  \begin{minipage}[h]{0.6\linewidth}
    \centering
    \hspace{-8mm}
    \includegraphics[scale=0.12]{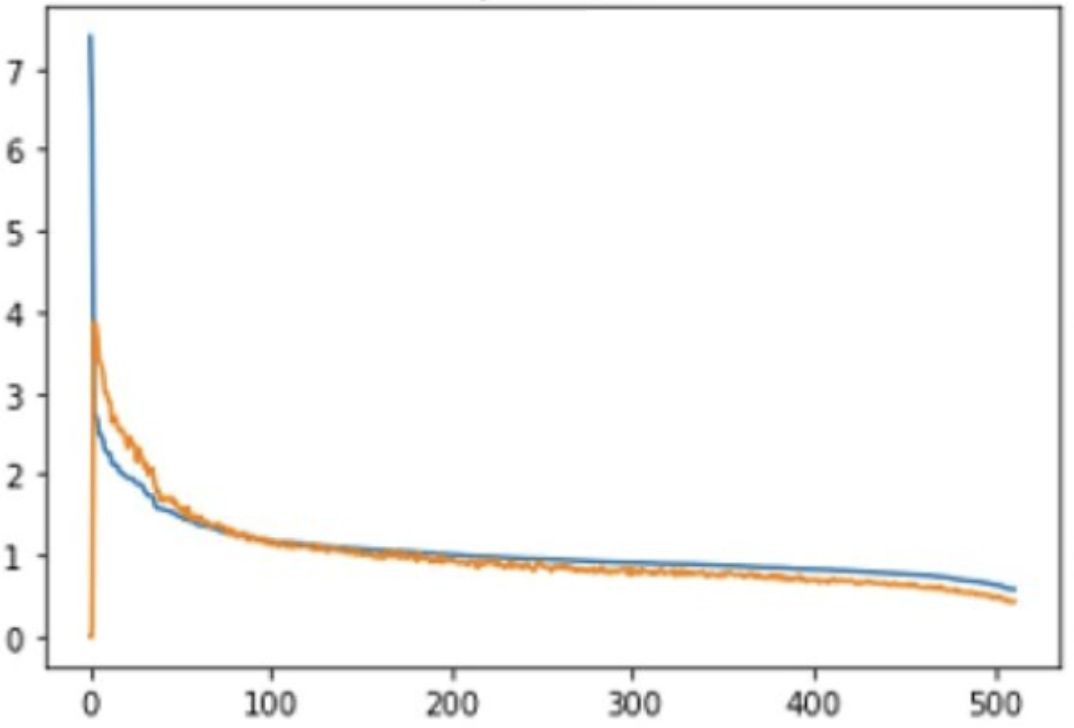}
  \end{minipage}
\caption{Comparison between $f(x)^{\top}f(x^-)$ and $ {P}(x^-|x)$ generated our method with $\epsilon = 0.3$ in different epoch. The blue line represents $e^{2f(x)^Tf(x)}$, while the orange line corresponds to $P(x^-|x) * B$. We noticed that the figure does not change significantly after 100 epochs, so we only report the comparisons at epochs 0 and 100. }\label{trans_inner}
\end{figure}

\begin{figure}[tb]

\begin{center}
{\includegraphics[width=9cm, height=6cm]{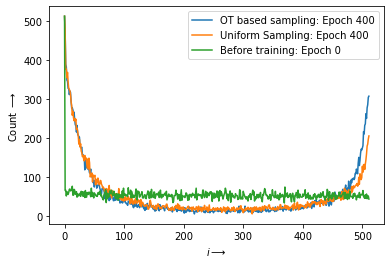}}
\end{center}
\caption{
%
Counts of number of samples (i.e., anchors) that have the same label as their $i$th closest neighbor within a minibatch of $512$ samples, for $i$ ranging from $1$ through $512$. Each count is calculated from the same minibatch of an epoch and minibatches are kept fixed across all epochs.
%
Green line: Before training. Orange line: after training with uniform sampling. Blue line: after training with our method.} \label{distribution} 
\end{figure}

To get additional insights, in Figure \ref{distribution} we plot the number of times the anchor has the same label as the negative sample, which are arranged in order of increasing distance from the anchor along the horizontal axis. The figure shows that not only are more samples getting closer to the anchor, but also that more samples with the same label are getting farther. This highlights the tradeoffs inherent to the design of hard negatives using the proposed method. The method forces the anchor to contrast more with images that are closer to it in the representation domain and these are likely to come from the same class as the anchor. 

We plot some examples of the most similar and dissimilar images from one batch in Figure~\ref{TSNE}. We observe that all of the top-10 similar images and nine of top-10 dissimilar images come from the class `plane'. The plane in the anchor image and the plane in the most similar images share some white-colored fuselage and darker background. However, the colors of the plane in the most dissimilar images are very different from that of the plane in the anchor image. Most of the dissimilar images contain black-colored fuselage and lighter backgrounds.

Thus, our model seems to have learned some irrelevant characteristics such as color and discarded higher-level features shared by the entire object class `plane' (the label of the anchor). This pushes some other plane images far away from the anchor in the representation space.

\begin{figure}[tb]
\begin{center}
{\includegraphics[width=9cm, height=7cm]{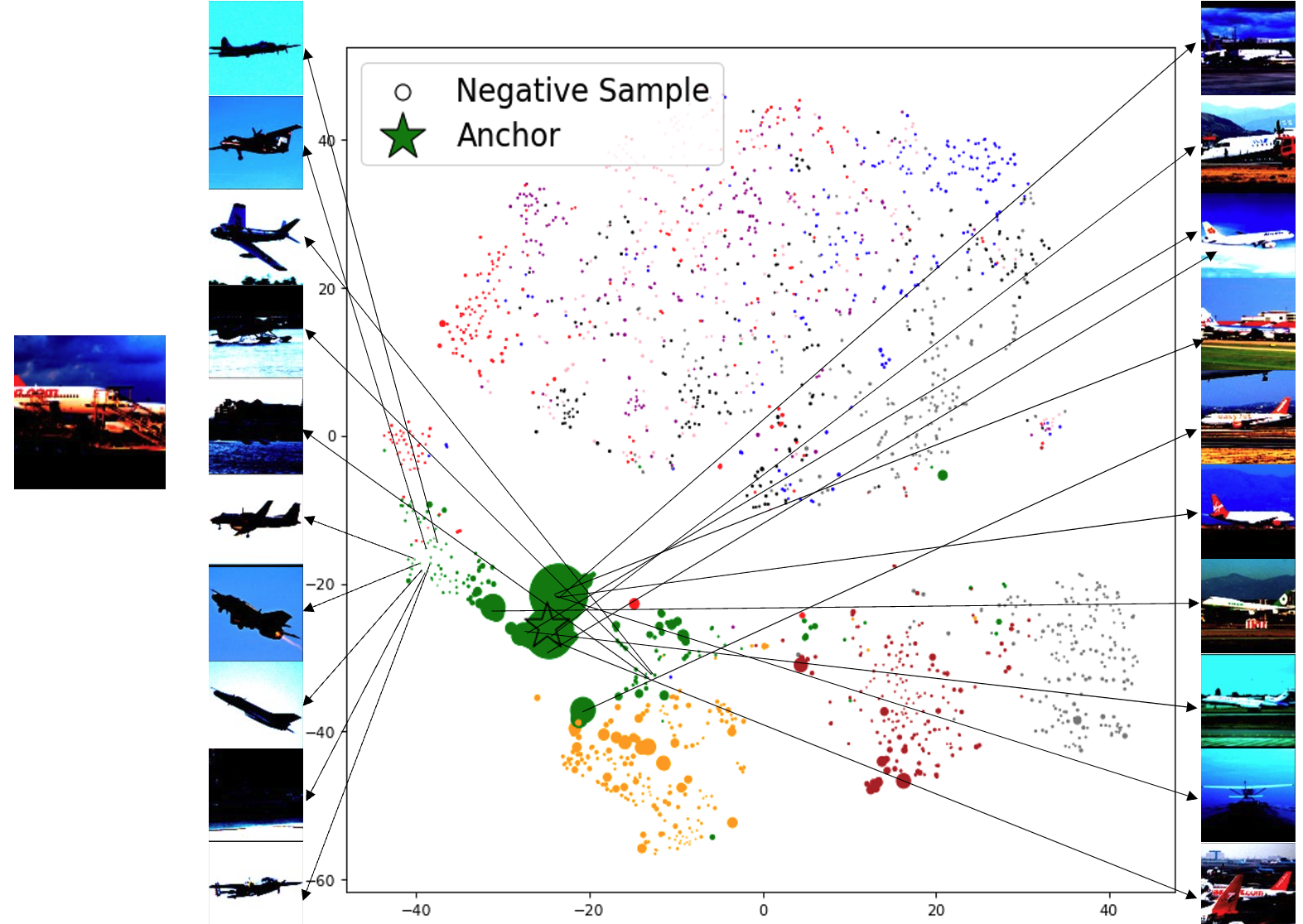}}
\end{center}
\caption{T-SNE visualization of representations learned with our method. The color indicates class label. Marker size corresponds to the calculated $P^*(x^-|x)$ with larger sizes for larger weights. We plot top-10 most similar and dissimilar images to the anchor on right and left separately among 2048 images from the test set.} \label{TSNE}

\end{figure}

\section{Conclusion and Future Work}
In this paper, we proposed and analyzed a novel min-max setting for hard negative sampling for unsupervised contrastive learning. We showed that without further regularization, the optimal representation can become degenerate for a large class of contrastive loss functions. We showed that reframing the problem in terms of regularized optimal transport (OT) offers a systematic approach for designing negative sampling distributions to learn good non-degenerate representations that can improve the performance of downstream classification tasks. Our work also provides a theoretical motivation for a state-of-the-art negative sampling method recently proposed in \cite{robinson2021contrastive}. 

The OT perspective on the design of hard-negatives proposed in this paper opens up the possibility of employing and investigating the effects of regularization mechanisms other than the entropic regularization of the optimal transport couplings, different marginal constraints, and the ground-costs, for potentially improved design of negative samples.

\bibliographystyle{IEEEtran}
\bibliography{IJCNN}

\end{document}